\documentclass[10pt,twocolumn,letterpaper]{article}

\usepackage{cvpr}              

\definecolor{cvprblue}{rgb}{0.21,0.49,0.74}
\usepackage[pagebackref,breaklinks,colorlinks,allcolors=cvprblue]{hyperref}
\usepackage[marginal]{footmisc}
\renewcommand{\thefootnote}{}
\usepackage{graphicx}
\usepackage{amsmath}
\usepackage{amssymb}
\usepackage{booktabs}
\usepackage{bm}
\usepackage[ruled,vlined]{algorithm2e}
\usepackage{color}
\usepackage{multirow}
\usepackage{amssymb}
\usepackage{makecell}
\usepackage[misc]{ifsym}

\usepackage{amsthm}
\usepackage{verbatim}
\usepackage[normalem]{ulem}
\useunder{\uline}{\ul}{}

\usepackage[capitalize]{cleveref}

\newtheorem{proposition}{Proposition}
\newtheorem{lemma}{Lemma}

\newtheorem{innercustomgeneric}{\customgenericname}
\providecommand{\customgenericname}{}
\newcommand{\newcustomtheorem}[2]{%
  \newenvironment{#1}[1]
  {%
   \ifdefined\crefalias\crefalias{innercustomgeneric}{#2}\fi
   \renewcommand\customgenericname{#2}%
   \renewcommand\theinnercustomgeneric{##1}%
   \innercustomgeneric
  }
  {\endinnercustomgeneric}%
  \ifdefined\crefname\crefname{#2}{#2}{#2s}\fi
}

\newcustomtheorem{customthm}{Theorem}
\newcustomtheorem{customprop}{Proposition}

\def\paperID{*****} 
\def\confName{CVPR}
\def\confYear{2025}

\title{Reversing Flow for Image Restoration}

\author{Haina Qin\textsuperscript{1,2*$\heartsuit$}, Wenyang Luo\textsuperscript{1,2*}, Libin Wang\textsuperscript{3$\dagger$}, Dandan Zheng\textsuperscript{3},\\ Jingdong Chen\textsuperscript{3}, Ming Yang\textsuperscript{3}, Bing Li\textsuperscript{1,2,4,$\textrm{\Letter}$} , Weiming Hu\textsuperscript{1,2} \\ 
\textsuperscript{1}State Key Laboratory of Multimodal Artificial Intelligence Systems (MAIS), CASIA;\\
\textsuperscript{2}School of Artificial Intelligence, University of Chinese Academy of Sciences;\ \ \textsuperscript{3}Ant Group;\\ \textsuperscript{4}PeopleAI Inc.\\
{\tt\small \{qinhaina2020@,luowenyang2020@,bli@nlpr.,wmhu@nlpr.\}ia.ac.cn lbin.wlb@antgroup.com}
}

\begin{document}
\maketitle
{\renewcommand{\thefootnote}{*}\footnotetext{Equal contribution. \href{https://github.com/luowyang/Defusion}{Project Page}. $\textrm{\Letter}$ Corresponding author.}}
{\renewcommand{\thefootnote}{$\heartsuit$}\footnotetext{Work done during internship at Ant Group. $\dagger$ Project lead.}}

\begin{abstract}
Image restoration aims to recover high-quality (HQ) images from degraded low-quality (LQ) ones by reversing the effects of degradation. Existing generative models for image restoration, including diffusion and score-based models, often treat the degradation process as a stochastic transformation, which introduces inefficiency and complexity. In this work, we propose ResFlow, a novel image restoration framework that models the degradation process as a deterministic path using continuous normalizing flows. ResFlow augments the degradation process with an auxiliary process that disambiguates the uncertainty in HQ prediction to enable reversible modeling of the degradation process. ResFlow adopts entropy-preserving flow paths and learns the augmented degradation flow by matching the velocity field. ResFlow significantly improves the performance and speed of image restoration, completing the task in fewer than four sampling steps. Extensive experiments demonstrate that ResFlow achieves state-of-the-art results across various image restoration benchmarks, offering a practical and efficient solution for real-world applications.

\end{abstract}


    
\vspace{-0.2cm}
\section{Introduction}
\label{sec:intro}


Image restoration \cite{wang2018esrgan, pan2021exploiting, kim2022bigcolor, wang2022zero, fei2023generative, lin2023diffbir} refers to recovering high-quality (HQ) images from degraded, low-quality (LQ) ones by reversing the effects associated with image degradation to reconstruct the original object. Many problems in image restoration, such as dehaze \cite{guo2022image, zhou2021learning, chen2023learning}, weather removal \cite{deng2020detail, wang2023smartassign, zhu2023learning, liu2018desnownet, ozdenizci2023restoring}, denoise \cite{moran2020noisier2noise, lin2023unsupervised, zhang2023mm}, and artifact removal \cite{jiang2021towards, ehrlich2020quantization}.

\begin{figure}[t]
    \centering
    \includegraphics[width=\linewidth]{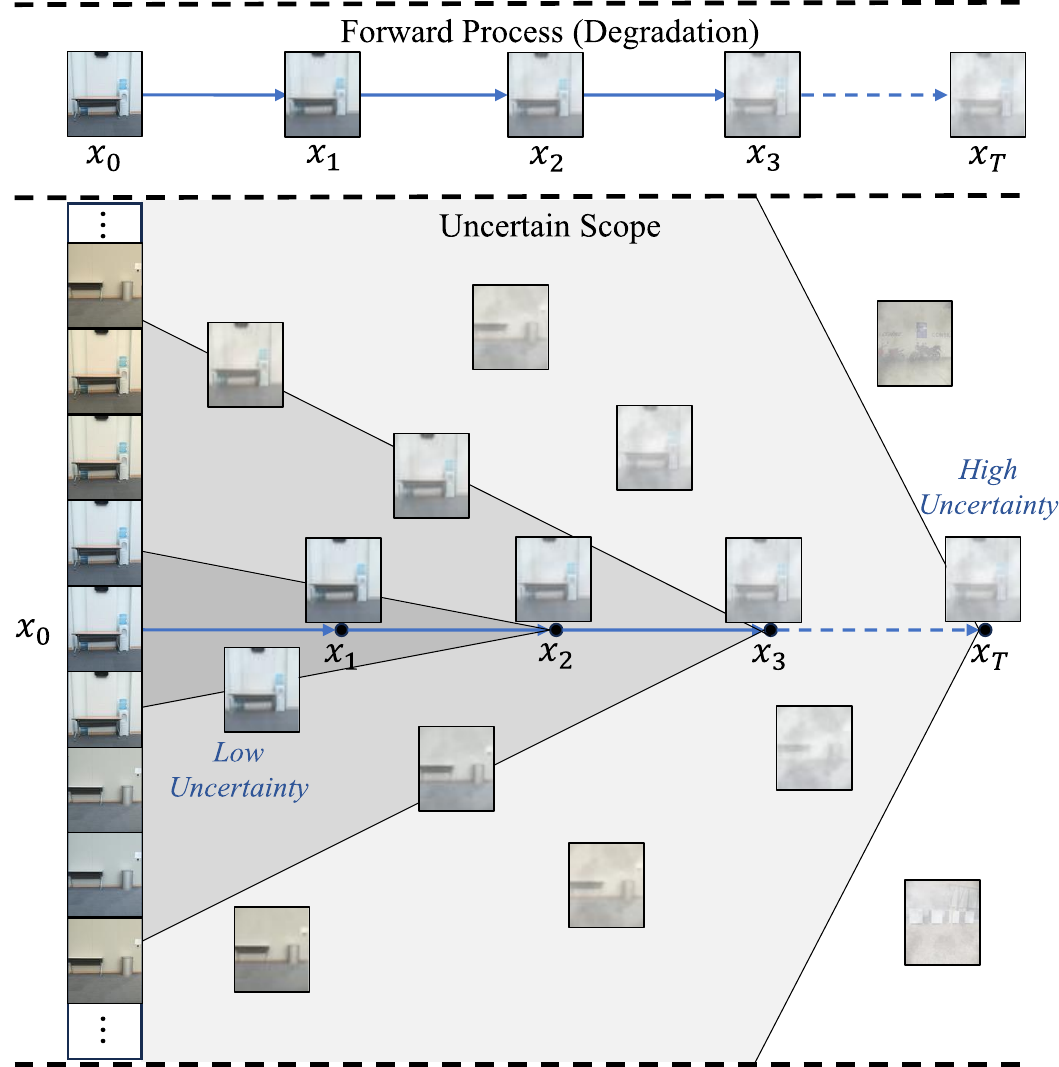}
    \caption{Image restoration is an ill-posed problem. The degradation process incurs decreasing mutual information between HQ and intermediate images, and multiple HQ images can degrade to similar or the same LQ images when their variations diminish in the process. Gray regions represent the uncertainty range from an intermediate state. As the image moves from LQ to HQ, its mutual information with HQ increases, and the uncertainty scope shrinks.}
    \label{fig:intro}
    \vspace{-0.4cm}
\end{figure}

Image restoration is ill-posed because it erases information in the HQ images. Given an LQ image, its corresponding HQ image is not necessarily unique, as depicted in \cref{fig:intro}. Image degradation gradually removes details from the HQ images. The corrupted image can result from multiple HQ images with different degradation processes. In fact, degradation is a Markov chain that transits from HQ to LQ images, which is subject to the data processing inequality (DPI) \cite{beaudry2011intuitive}: the mutual information between HQ and intermediate images decreases as more distortion is applied to the image.
The ambiguous multi-correspondence between HQ and LQ images creates \textit{uncertainty scopes} for LQ images defined as the collection of possible HQ images given the LQ images. The uncertainty scope should be disambiguated to make the degradation process reversible by selecting exactly one HQ from the uncertainty scope given the intermediate image.

Generally, restoration methods navigate the uncertainty scope with prior knowledge besides the LQ images to reverse the degradation effects. The natural image distribution from diffusion and score-based generative models \cite{sohl2015deep, ho2020denoising, song2020score} serves as a strong prior for many approaches \cite{li2022srdiff, shang2024resdiff, xia2023diffir}. These methods define degradation as a conditional stochastic process that stochastically diffuses the HQ image into noise and learns a score function conditioned on the LQ image to reverse this process. The reverse process generates the HQ images under the guidance of the LQ images, where the LQ images provide structural and semantic hints. However, starting the reverse process from Gaussian noise is unnecessary and inefficient be cause it regenerates the structures already known in the LQ images.

To address this inefficiency, some studies incorporate prior knowledge of the degraded image directly into the forward stochastic process to enhance the efficiency of the reverse process \cite{kawar2022denoising, luo2023image, luo2023refusion, shi2023resfusion, liu20232, liu2024residual, yue2024resshift}. For example, DDRM \cite{kawar2022denoising} progressively denoises samples stochastically to achieve the desired output. IR-SDE \cite{luo2023image} models the degradation process using mean-reverting stochastic differential equations (SDEs), while I2SB \cite{liu20232} constructs a Schr{\"o}dinger bridge between the clean and degraded data distributions. ResShift \cite{yue2024resshift} transfers residuals from degraded low-resolution images to high-resolution ones for restoration in the latent space. RDDM \cite{liu2024residual} and Resfusion \cite{shi2023resfusion} introduce residual terms in the forward process. However, these approaches still treat the degradation process as a progressively diffusing stochastic forward process, which seems unnecessary and introduces additional complexity and inefficiency. Given that the degraded image is already known, the degradation process could be redefined as a deterministic forward process.

In this paper, we propose ResFlow, a novel general framework that reverses the deterministic paths between HQ and LQ images for image restoration. ResFlow models the forward process as a deterministic continuous normalizing flow \cite{chen2018neural, lipman2022flow}, directly simulating the path from high-quality images to low-quality ones. The image restoration process begins directly from the degraded image in the reverse process. The construction of the reverse process is non-trivial due to the existence of uncertainty scope: multiple plausible clear images may correspond to the same degraded image. We take into consideration the decrease of mutual information during degradation and augment the degradation process with an auxiliary process that couples with the uncertainty scope to disambiguate the velocity of the backward process. We also derive a flow path based on the intuition of entropy conservation in reversible processes. The deterministic flow path allows ResFlow to achieve better performance and faster generation speeds, completing image restoration in fewer than four sampling steps.

Our contributions are summarized as follows:
\begin{itemize}
    \item We present ResFlow, a novel image restoration framework that reverses the deterministic degradation path between HQ and LQ images for image restoration, which achieves better performance and faster inference.
    \item ResFlow reverses image degradation by augmenting the degradation process with an auxiliary process that couples with the uncertainty scope, and adopts an entropy-preserving flow path in the reverse process.
    \item We conduct experiments on various image restoration tasks and datasets. Results demonstrate the effectiveness of ResFlow that sets new state-of-the-art performance.
\end{itemize}

\section{Related Work}
\label{sec:related}

\begin{figure*}[htbp]
    \centering
    \includegraphics[width=\textwidth]{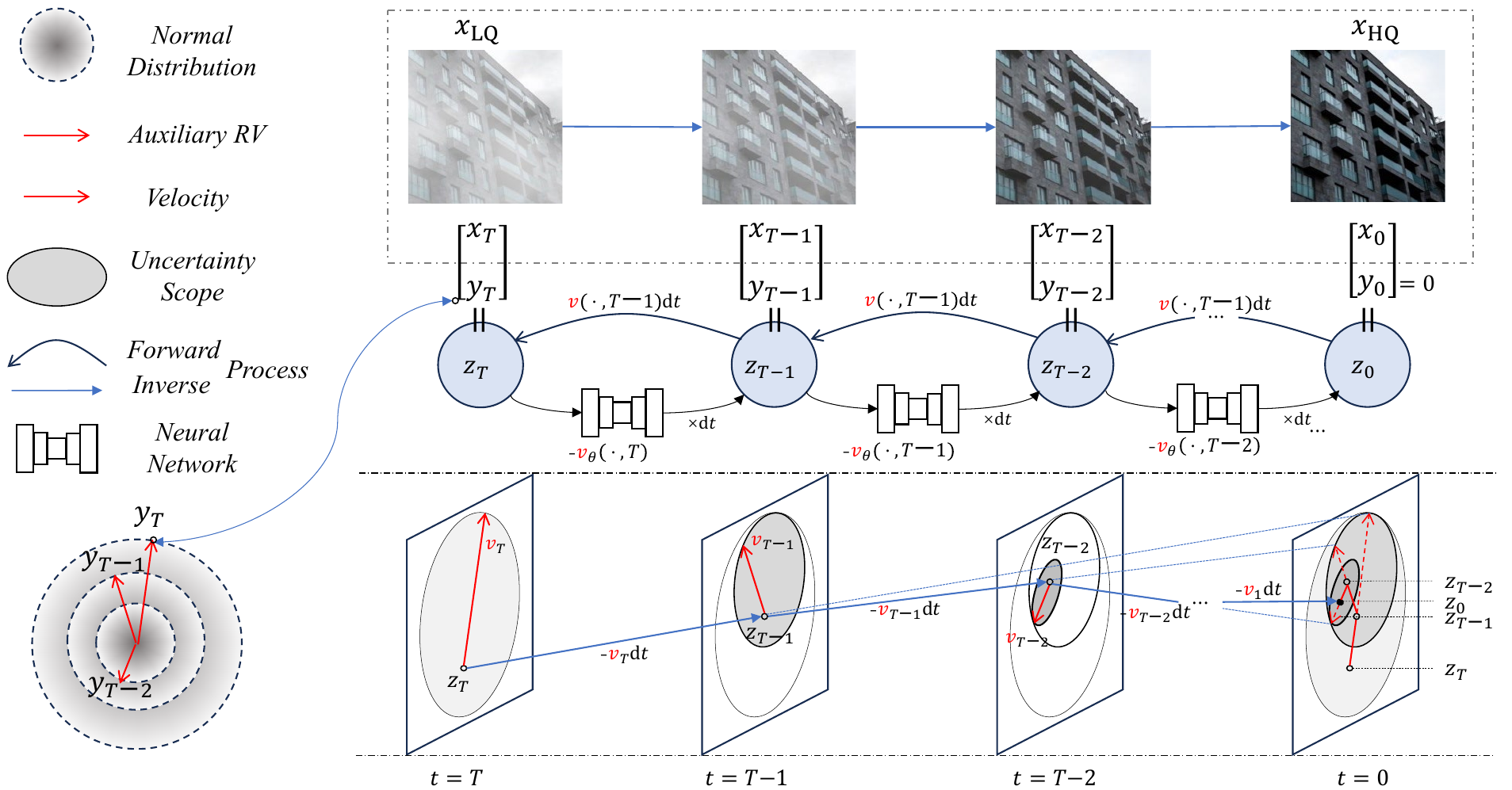}
    \caption{Framework of ResFlow. \textit{RV} stands for \textit{random variable}. The state $\bm{z}_t$ consists of a data component $\bm{x}_t$ that transits between HQ and LQ images and an auxiliary component $\bm{y}_t$ that disambiguates the velocity to ensure invertibility. The forward process is defined by interpolation, while the reverse process is learned by matching the velocity field. The lower part depicts the transition of ResFlow. Image degradation is usually non-reversible due to decreasing mutual information; thus, the velocity is uncertain for $\bm{x}_t$. The range reachable by possible velocity is dubbed uncertainty scope. As $\bm{x}_t$ approaches $\bm{x}_0$, the uncertainty in estimating velocity decreases, and so does the typical set of $\bm{y}_t$.}
    \label{fig:framework}
\end{figure*}

Image restoration can be formulated as an inverse problem \cite{banham1997digital, wang2022zero} where a high quality (HQ) image is reconstructed given its degraded 
 version (rainy \cite{deng2020detail, wang2023smartassign, zhu2023learning}, snowy \cite{liu2018desnownet, ozdenizci2023restoring}, hazy \cite{guo2022image, zhou2021learning, chen2023learning}, noisy \cite{moran2020noisier2noise, lin2023unsupervised, zhang2023mm}, compressed \cite{jiang2021towards, ehrlich2020quantization}, etc.) dubbed low quality (LQ) image.
Image degradations usually erase information from HQ images; thus, the inverse problem (i.e., image restoration) is ill-posed in the sense that one LQ image may come from multiple HQ images by different degradations \cite{banham1997digital, pan2021exploiting, wang2022zero}. To tackle this problem, generative models are widely adopted to provide prior knowledge \cite{pan2021exploiting, kim2022bigcolor, wang2022zero, fei2023generative, lin2023diffbir} about the distribution of HQ images.

Early works follow generative adversarial nets (GANs) \cite{goodfellow2020generative, mirza2014conditional} and train a discriminator to guide the LQ images towards the distribution of HQ images \cite{ledig2017photo, wang2018esrgan, wang2021real, wang2021towards}. Variational autoencoders (VAEs) \cite{kingma2013auto} are another choice that optimizes the evidence lower bound (ELBO) of the restored images \cite{zheng2022learn, soh2022variational, lin2023catch}. These approaches predict the HQ images in a single step and suffer from the quality-diversity dilemma \cite{xiao2021tackling}: the adversarial training of GANs causes unstable training and mode collapse problem \cite{arjovsky2017wasserstein, gulrajani2017improved, salimans2016improved, mao2017least}. In contrast, the mean-squared error (MSE) in VAE training leads to blurry and low-quality predictions \cite{blau2018perception, zhao2016loss, zhang2018unreasonable}.

Recent methods perform multi-step or iterative prediction \cite{liu2024residual, delbracio2023inversion, luo2023refusion, liu2024diff}: splitting the degradation process into small steps and investing each step is more tractable and easier to learn.
Many works build upon diffusion models \cite{sohl2015deep, ho2020denoising, song2020score} or Schr{\"o}dinger bridges \cite{liu20232} and model degradation processes as stochastic paths that transit between HQ and LQ images \cite{liu2024residual, shi2023resfusion, yue2024resshift, luo2023refusion, xia2023diffir, liu2024diff, zhang2024diffusion, yu2024scaling, liu2024structure, kawar2022denoising, luo2023image}. The uncertainty introduced by random transitions slows down training and inference. 
Other works \cite{helminger2021generic, lugmayr2022normalizing, wang2022low} approximate the degradations as a sequence of reversible deterministic steps, dubbed normalizing flows \cite{papamakarios2021normalizing, kingma2018glow}. Taking the limit of the number of steps and making each step infinitesimally small, one obtains an ordinary differential equation (ODE), also known as a continuous normalizing flow \cite{chen2018neural, onken2021ot, lipman2022flow}.
InDI \cite{delbracio2023inversion} incrementally estimates the HQ images, which is equivalent to solving a ``residual flow'' that becomes sensitive to error in prediction near HQ, leading to sub-optimal results.
These flow-based methods contradict the ill-posed and irreversible nature of the image degradation, leading to inferior performance.

Our work defines the degradations from HQ to LQ images by deterministic paths. Acknowledging the decrease of mutual information in image degradation that results in uncertain correspondence between LQ and HQ images, we augment the degradation process with an \textit{auxiliary process} that couples with the uncertainty scope and guides the restoration directions. Our augmented degradation process is fully reversible and can be modeled by a deterministic ODE dubbed \textit{degradation flow}. We learn the velocity field of the degradation flow by matching it with the ground-truth flow path \cite{lipman2022flow}, which leads to fast training and inference and achieves better restoration.

\section{Method}
\label{sec:method}

\subsection{Reversing Flow for Image Restoration}
\label{subsec:visual_instructions}

We tackle the image restoration problem by \textit{inversion}: learning the degradation process that corrupts a high-quality (HQ) image $\bm{x}_\text{HQ}$ into a low-quality (LQ) image $\bm{x}_\text{LQ}$, and ``reversing'' the learned process to restore the HQ image from the LQ image.
The learned degradation process should be 1) \textit{reversible} to allow inverting LQ images to HQ images and 2) \textit{tractable} to enable efficient training and inference.

A natural choice is to model the degradation process by an ordinary differential equation (ODE) on random process $\{\bm{z}_t \,\vert\, 0 \le t \le 1\}$:
\begin{equation}
    \frac{\partial\bm{z}_t}{\partial t} = \bm{v}(\bm{z}_t, t); \quad 0 \le t \le 1,
    \label{eq:ode}
\end{equation}
where $\bm{v}$ is the velocity field, $\bm{z}_0$ corresponds to the HQ image, and $\bm{z}_1$ to the LQ image.
\cref{eq:ode} is also known as a continuous normalizing flow in the literature \cite{chen2018neural, lipman2022flow}.

However, \cref{eq:ode} cannot directly apply to image degradation because it describes an reversible process. In contrast, image degradation is generally irreversible. This can be illustrated from the perspective of mutual information: a random process described by \cref{eq:ode} preserves mutual information while image degradation does not. First, we have the following proposition.

\begin{proposition}
Given random process ${\bm{z}_t}$ defined by \cref{eq:ode}, denote the mutual information as $\mathrm{MI}(\cdot, \cdot)$, then for any reference random variable $\bm{r}$ and any $0\le t_1, t_2 \le 1$, we have
\begin{equation}
    \mathrm{MI}(\bm{z}_{t_1}, \bm{r}) = \mathrm{MI}(\bm{z}_{t_2}, \bm{r}).
    \label{eq:prop1}
\end{equation}
\label{prop:mi}
\end{proposition}

\begin{proof}
    See the supplementary material.
\end{proof}

The reference random variable $\bm{r}$ can be arbitrary state $\bm{z}_t \neq \bm{z}_0$, indicating that simulating \cref{eq:ode} does not lose any information: everything we known about $\bm{z}_0$ remains in $\bm{z}_t$.
However, this property does not hold for the ill-posed image degradation. Generally, the mutual information between the intermediate states and the HQ images decreases as more distortion is applied to the images as a consequence of the data processing inequality (DPI) \cite{beaudry2011intuitive}: as $\bm{x}_t$ approaches $\bm{x}_\text{LQ}$, it shares less mutual information with $\bm{x}_\text{HQ}$, so generally less is known about HQ images.
As a result, an LQ image can correspond to multiple HQ images through different degradation processes. \cref{fig:intro} shows an example where adding haze to different HQ images blends their content and leads to similar or the same LQ images. Noise is another example that eventually converts HQ images into indistinguishable noise.

The collection of possible HQ images given the LQ images is dubbed the uncertainty scope. The uncertainty scope should be disambiguated to make the degradation process reversible by selecting exactly one HQ from the uncertainty scope given the intermediate image. We achieve this by augmenting the degradation process with an auxiliary process $\{\bm{y}_t \,\vert\, 0 \le t \le 1\}$ that couples with the uncertainty scope and evolves with the degraded image. The ODE states $\{\bm{z}_t\}$ now become
\begin{equation}
    \bm{z}_t^\mathsf{T} = [\bm{x}_t^\mathsf{T}; \bm{y}_t^\mathsf{T}], \quad \\
    \bm{z}_0^\mathsf{T} = [\bm{x}_\text{HQ}^\mathsf{T}; \bm{y}_0^\mathsf{T}], \quad \\
    \bm{z}_1^\mathsf{T} = [\bm{x}_\text{LQ}^\mathsf{T}; \bm{y}_1^\mathsf{T}].
    \label{eq:augment}
\end{equation}
Conceptually, $\bm{y}_t$ encodes the ``information loss'' caused by image degradation. According to DPI, the mutual information from the coupling between it and the HQ images increases as $\bm{z}_t$ approaches $\bm{z}_1$ to keep $\mathrm{MI}(\bm{z}_t, \bm{z}_0)$ constant. This makes the choice of $\{\bm{y}_t\}$ non-trivial because $\bm{y}_1$ should have the maximal mutual information with $\bm{x}_0$, which is equivalent to knowing the ground-truth $\bm{x}_0$ in prior and thus infeasible. The next section details how we parameterize the augmented flow and tackle this problem.

\subsection{Parameterization}
\label{subsec:parameterization}

Learning to invert the degradation process by the augmented flow (\cref{eq:ode,eq:augment}) requires obtaining the coupling between the auxiliary $\bm{y}_t$ and $\bm{x}_0$ that conceptually maps to the degraded mutual information.
While one can manually determine such coupling, it is not necessarily optimal.
Instead, inspired by recent advances in transport-based generative modeling techniques \cite{liu2022flow,lipman2022flow}, we \textit{learn the deterministic coupling starting from an arbitrary coupling between $\bm{y}_t$ and $\bm{x}_0$}.

Given a pair of HQ and LQ images, the only bound of the degradation process (\cref{eq:ode}) is that it begins with the HQ image and ends with the LQ image as formally given by \cref{eq:augment}. Usually, the ground-truth ``natural'' degradation paths between HQ and LQ images are infeasible to acquire,
so we instead define the flow paths of $\bm{z}_t$ as geodesics in the Euclidean space:
\begin{align}
    &\bm{z}_t = \alpha_t \bm{z}_0 + \sigma_t \bm{z}_1, \quad \alpha_t, \sigma_t \ge 0,  \label{eq:interpolation_1} \\
    &\text{such that} ~~ \alpha_0 = \sigma_1 = 1, ~~ \alpha_1 = \sigma_0 = 0.
    \label{eq:interpolation}
\end{align}
We name $\{\alpha_t, \sigma_t\}$ as the degradation schedule defining the flow paths' dynamics.
Given the degradation schedules, the augmented degradation process \cref{eq:ode,eq:augment} can be parameterized by a neural network $\bm{v}_\theta$ that estimates the velocity field $\bm{v}$:
\begin{equation}
    \frac{\partial[\bm{x}_t^\mathsf{T}; \bm{y}_t^\mathsf{T}]^\mathsf{T}}{\partial t} = \bm{v}_\theta(\bm{x}_t, \bm{y}_t, t).
    \label{eq:parameterization}
\end{equation}

The auxiliary $\bm{y}_t$ is chosen to be Gaussian at $t=1$ as it maximizes entropy and zero at $t=0$ where the restoration ends.
During training, $\{\bm{y}_t\}$ are independently coupled with both $\bm{x_0}$ and $\bm{x}_1$. However, the trained velocity network $\bm{v}_\theta$ induces a deterministic coupling between $\bm{x}_0$, $\bm{x}_1$ and $\bm{y}_t$ as defined by \cref{eq:parameterization}. This coupling disambiguates the velocity when multiple possible HQ images exist given an $\bm{x}_t$ as in \cref{fig:intro,fig:framework}. In such cases, the velocity network $\bm{v}_\theta$ uniquely maps $\bm{y}_t$ to one of the possible velocities.

We also note that the image $\bm{x}_t$ and auxiliary $\bm{y}_t$ in $\bm{z}_t$ do not necessarily have to use the same degradation schedule.
Denoting the respective degradation schedules of $\bm{x}_t$ and $\bm{y}_t$ as $\{\alpha_t^{\bm{x}}, \sigma_t^{\bm{x}}\}$ and $\{\alpha_t^{\bm{y}}, \sigma_t^{\bm{y}}\}$, we propose the following entropy-preserving degradation schedule:
\begin{align}
    \alpha_t^{\bm{x}} &= 1 - \sigma_t^{\bm{x}},  \quad \sigma_t^{\bm{x}} = t \\ 
    \alpha_t^{\bm{y}} &= 1 - \sigma_t^{\bm{y}} ,\quad \sigma_t^{\bm{y}} = \beta \cdot \left( 1 - t + \beta \right)^{-1}.
    \label{eq:schedule}
\end{align}
Here, $\beta=10$ is a hyperparameter. This schedule moves $\bm{x}_t$ in straight lines (geodesics in Euclidean space). It keeps the entropy of $\bm{z}_t$ constant throughout the flow paths based on the intuition that the entropy remains the same for the reversible process
(see the supplementary material for the derivation).

An alternative parameterization is to estimate the expectation $\mathbb{E}\left[\bm{z}_0 \,\vert\, \bm{z}_t\right]$ with a neural network \cite{delbracio2023inversion}. However, this approach is equivalent to a time-weighted version of \cref{eq:parameterization} that leads to high discretization errors near $t=0$ when solved numerically.

\begin{table*}[htbp]
\caption{Synthetic datasets. Desnowing, Deraining, and Dehazing results on Snow100K \cite{liu2018desnownet}, Outdoor-Rain \cite{li2019heavy}, and Dense-Haze \cite{ancuti2019dense}. }
\label{table:synthetic}
\renewcommand\arraystretch{1.30}
\resizebox{\textwidth}{!}{%
\begin{tabular}{lccclccclcc}
\hline
\multirow{2}{*}{\textbf{Method}} & \multicolumn{3}{c}{\textbf{Desnowing}} & \multirow{2}{*}{\textbf{Method}} & \multicolumn{3}{c}{\textbf{Deraining}} & \multirow{2}{*}{\textbf{Method}} & \multicolumn{2}{c}{\textbf{Dehazing}} \\ \cline{2-4} \cline{6-8} \cline{10-11} 
 & PSNR$\uparrow$ & SSIM$\uparrow$ & LPIPS$\downarrow$ &  & PSNR$\uparrow$ & SSIM$\uparrow$ & LPIPS$\downarrow$ &  & PSNR$\uparrow$ & SSIM$\uparrow$ \\ \hline
SPANet\cite{wang2019spatial} & 23.70 & 0.793 & 0.104 & CycleGAN\cite{zhu2017unpaired} & 17.62 & 0.656 & - & DehazeNet\cite{cai2016dehazenet} & 13.84 & 0.43 \\
JSTASR\cite{chen2020jstasr} & 25.32 & 0.807 & 0.059 & pix2pix\cite{isola2017image} & 19.09 & 0.710 & - & AOD-Net\cite{li2017aod} & 13.14 & 0.41 \\
RESCAN\cite{li2018recurrent} & 26.08 & 0.810 & 0.054 & HRGAN\cite{li2019heavy} & 21.56 & 0.85 & 0.154 & SGID\cite{bai2022self} & 12.49 & 0.51 \\
DesnowNet\cite{liu2018desnownet} & 27.17 & 0.898 & 0.070 & PCNet\cite{jiang2021rain} & 26.19 & 0.901 & 0.132 & MSBDN\cite{dong2020multi} & 15.37 & 0.49 \\
MPRNet\cite{zamir2021multi} & 29.76 & 0.894 & 0.049 & MPRNet\cite{zamir2021multi} & 28.03 & 0.919 & 0.089 & FFA-Net\cite{qin2020ffa} & 14.39 & 0.45 \\
NAFNet\cite{chen2022simple} & 30.06 & 0.901 & 0.051 & NAFNet\cite{chen2022simple} & 29.59 & 0.902 & 0.085 & AECR-Net\cite{wu2021contrastive} & 15.80 & 0.47 \\
Restormer\cite{zamir2022restormer} & 30.52 & 0.909 & 0.047 & Restormer\cite{zamir2022restormer} & 29.97 & 0.921 & 0.074 & DeHamer\cite{guo2022image} & {\ul 16.62} & 0.56 \\
SnowDiff64\cite{ozdenizci2023restoring} & 30.43 & {\ul 0.914} & 0.035 & RainHazeDiff64\cite{ozdenizci2023restoring} & 28.38 & 0.932 & 0.067 & PMNet\cite{ye2022perceiving} & 16.79 & 0.51 \\
SnowrDiff128\cite{ozdenizci2023restoring} & 30.28 & 0.900 & 0.038 & RainHazeDiff128\cite{ozdenizci2023restoring} & 26.84 & 0.915 & 0.071 & FocalNet\cite{cui2023focal} & 17.07 & \textbf{0.63} \\
DTPM-4\cite{ye2024learning} & {\ul 30.92} & \textbf{0.917} & {\ul 0.034} & DTPM-4\cite{ye2024learning} & {\ul 30.99} & {\ul 0.934} & {\ul 0.0635} & MB-Taylor\cite{qiu2023mb} & 16.44 & 0.56 \\ \hline
\textbf{ResFlow(Ours)} & \textbf{31.86} & \textbf{0.917} & \textbf{0.030} & \textbf{ResFlow(Ours)} & \textbf{32.82} & \textbf{0.936} & \textbf{0.0514} & \textbf{ResFlow(Ours)} & \textbf{17.12} & {\ul 0.59} \\ \hline
\end{tabular}%
}
\end{table*}

\subsection{Optimization and Inference}
\label{subsec:optimization}

The parameterized augmented degradation ODE \cref{eq:parameterization} is learned by matching the velocity field \cite{lipman2022flow, liu2022flow}:
\begin{align}
    & \min_{\theta} \mathbb{E}_{\bm{x}_0, \bm{x}_1, \bm{y}_0, \bm{y}_1} \left[ \int_{0}^{1} \lambda(t) \left\| \bm{v}_\theta(\bm{x}_t, \bm{y}_t, t) - \dot{\bm{z}}_t \right\|^2 \,\mathrm{d}t \right], \label{eq:optimization_inner} \\
    & \text{where} \quad \dot{\bm{z}}_t = 
    \left[\begin{matrix}
        \dot{\bm{x}}_t \\
        \dot{\bm{y}}_t \\
    \end{matrix}\right]
    =
    \left[\begin{matrix}
        \dot{\alpha}_t^{\bm{x}} \bm{x}_0 + \dot{\sigma}_t^{\bm{x}} \bm{x}_1 \\
        \dot{\alpha}_t^{\bm{y}} \bm{y}_0 + \dot{\sigma}_t^{\bm{y}} \bm{y}_1 \\
    \end{matrix}\right].
    \label{eq:optimization}
\end{align}
Here, $\lambda(t)$ is a time-dependent loss weighting function, $\|\cdot\|$ is the L2 norm. Optimizing \cref{eq:optimization} is efficient as it does not require simulating \cref{eq:parameterization} as in traditional simulation-based methods \cite{chen2018neural, helminger2021generic, lugmayr2022normalizing}.
Moreover, \cref{eq:optimization} provides additional benefit that any convex transport cost induced by the coupling between $\bm{z}_0$ and $\bm{z}_1$ is guaranteed to be non-increasing \cite{liu2022flow} combined with \cref{eq:schedule}.

We also propose a loss weighting scheme for image degradation that emphasizes time $t$ close to $1$, which empirically improves the quality of estimated HQ images:
\begin{equation}
    \lambda(t) = \left(\cos\left(\frac{\pi}{2} (t - 2)\right) + 1\right)^{\gamma},
    \label{eq:weighting}
\end{equation}
where $\gamma=1.75$. The rationale is that when $\bm{x}_t$ is close to the LQ image $\bm{x}_1$, its mutual information with the HQ image $\bm{x}_0$ decreases, making the velocity prediction more difficult near $t=1$. Thus, we increase the loss weighting as $t$ increases to balance the gradients towards those from the harder tasks.

After training by \cref{eq:optimization}, we restore the LQ image $\bm{x}_1$ by first sampling $\bm{y}_1 \sim \mathcal{N}(\bm{0}, \bm{I})$, then numerically solving \cref{eq:parameterization} from $t=1$ to $t=0$ and obtain $\hat{\bm{x}}_0$ as the predicted HQ image. The intermediate $\hat{\bm{y}}_t$ is discarded and replaced with the ground-truth $\bm{y}_t$ given by \cref{eq:interpolation}. Conceptually, $\bm{y}_t$ is asymptotically mapped to the corrupted information during inference.

\section{Experiments}
\label{sec:exp}

\begin{figure*}[htbp]
    \centering
    \includegraphics[width=\textwidth]{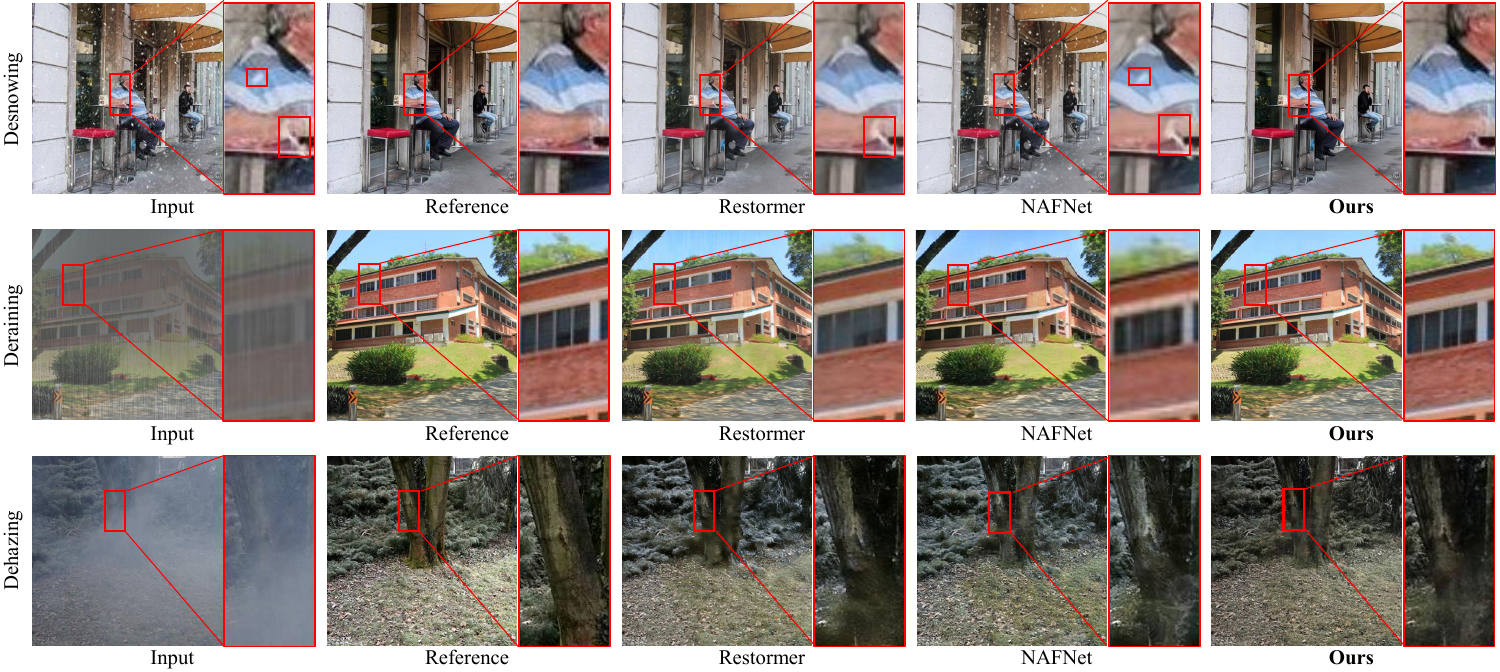}
    \caption{Dehazing, Deraining, and Desnowing results. The part of the image is methodized to observe the local details clearly. From left to right: input blurry images, reference images, and the predicted images obtained by Restormer \cite{zamir2022restormer}, NAFNet \cite{chen2022simple}, and our ResFlow, respectively. Note that both Restormer and NAFNet retain some artifacts in the output images.}
    \label{fig:weather}
\end{figure*}

\subsection{Settings}
To evaluate ResFlow's performance, we conducted experiments on five major image restoration tasks, including desnowing, draining, dehazing, denoising, and JPEG compression artifact removal, using synthetic and real-world datasets. For image desnowing, we conducted experiments on the Snow100K \cite{liu2018desnownet} dataset and the RealSnow \cite{zhu2023learning} dataset. We used the Outdoor-Rain \cite{li2019heavy} dataset and the LHP \cite{guo2023sky} dataset to evaluate deraining performance. The Dense-Haze \cite{ancuti2019dense} and NH-HAZE \cite{ancuti2020nh} datasets were employed to assess image dehazing. We utilized the SIDD \cite{abdelhamed2018high} dataset for real-world denoising. The DPDD \cite{abuolaim2020defocus} dataset was used to test single-image defocus deblurring. The LVE1 \cite{sheikh2005live} and BSD500 \cite{martin2001database} datasets were used to verify the removal of JPEG artifacts through entropy validation. We computed several distortion and perception-based metrics, including PSNR, SSIM \cite{wang2004image}, MAE, and LPIPS \cite{zhang2018unreasonable}.

We adopted the same U-Net architecture \cite{ronneberger2015u} as DDPM \cite{ho2020denoising} to predict the velocity in \cref{eq:ode} for all tasks. Timestep $t$ is embedded and injected into U-Net blocks via adaptive layer normalization \cite{xu2019understanding}. 
The model was trained using the Adam optimizer \cite{kingma2014adam} on 256-resolution image crops and tested on full-resolution images. The learning rate and other hyperparameters are detailed in Appendix C. We employed a uniform time schedule and performed only four sampling steps for all datasets.

\subsection{Main Results}

\begin{table*}[!h]
\caption{Real-world datasets. Dehazing results on NH-HAZE 
\cite{ancuti2020nh}, Denoising results on SIDD \cite{abdelhamed2018high}, Deraining results on LHP \cite{guo2023sky}, and Desnowing results on RealSnow \cite{zhu2023learning}}
\renewcommand\arraystretch{1.30}
\resizebox{\textwidth}{!}{%
\begin{tabular}{lcclcclcclcc}
\hline
\multirow{2}{*}{\textbf{Method}} & \multicolumn{2}{c}{\textbf{Denoising}} & \multirow{2}{*}{\textbf{Method}} & \multicolumn{2}{c}{\textbf{Dehazing}} & \multirow{2}{*}{\textbf{Method}} & \multicolumn{2}{c}{\textbf{Deraining}} & \multirow{2}{*}{\textbf{Method}} & \multicolumn{2}{c}{\textbf{Desnowing}} \\ \cline{2-3} \cline{5-6} \cline{8-9} \cline{11-12} 
 & PSNR & SSIM &  & \multicolumn{1}{l}{PSNR} & \multicolumn{1}{l}{SSIM} &  & PSNR & SSIM &  & PSNR & SSIM \\ \hline
DAGL\cite{mou2021dynamic} & 38.94 & 0.953 & DehazeNet\cite{cai2016dehazenet} & 16.62 & 0.52 & SPANet\cite{wang2019spatial} & 31.19 & 0.934 & MIRNetv2\cite{zamir2022learning} & 31.39 & 0.916 \\
DeamNet\cite{ren2021adaptive} & 39.47 & 0.957 & AOD-Net\cite{li2017aod} & 15.40 & 0.57 & PReNet\cite{ren2019progressive} & 32.13 & 0.917 & ART\cite{zhang2022accurate} & 31.05 & 0.913 \\
MIRNet\cite{zamir2020learning} & 39.72 & 0.959 & AECR-Net\cite{wu2021contrastive} & 19.88 & 0.72 & RCDNet\cite{huang2021selective} & 32.34 & 0.915 & Restormer\cite{zamir2022restormer} & 31.38 & \textbf{0.923} \\
DANet\cite{yue2020dual} & 39.47 & 0.957 & DeHamer\cite{guo2022image} & 20.66 & 0.68 & MPRNet\cite{zamir2021multi} & 33.34 & 0.930 & NAFNet\cite{chen2022simple} & {\ul 31.44} & 0.919 \\
Restormer\cite{zamir2022restormer} & {\ul 40.02} & {\ul 0.960} & PMNet\cite{ye2022perceiving} & 20.42 & {\ul 0.73} & SCD-Former\cite{guo2023sky} & {\ul 34.33} & \textbf{0.946} & WGWS-Net\cite{zhu2023learning} & 31.37 & 0.919 \\
Xformer\cite{zhang2023xformer} & 39.98 & {\ul 0.960} & FocalNet\cite{cui2023focal} & {\ul 20.43} & \textbf{0.79} & IDT\cite{jiang2021towards} & 33.02 & 0.931 & TransWeather\cite{valanarasu2022transweather} & 31.13 & {\ul 0.922} \\ \hline
\textbf{ResFlow(Ours)} & \textbf{42.26} & \textbf{0.962} & \textbf{ResFlow(Ours)} & \textbf{21.44} & \textbf{0.79} & \textbf{ResFlow(Ours)} & \textbf{34.54} & {\ul 0.939} & \textbf{ResFlow(Ours)} & \textbf{31.63} & 0.919 \\ \hline
\end{tabular}%
}
\label{table:real}
\end{table*}

\begin{table*}[!h]
\caption{Single-image Defocus Deblurring comparisons on the DPDD \cite{abuolaim2020defocus} datasets.}
\label{table:dpdd}
\renewcommand\arraystretch{1.25}
\resizebox{\textwidth}{!}{%
\begin{tabular}{lcccccccccccc}
\hline
\multirow{2}{*}{\textbf{Method}} & \multicolumn{4}{c}{\textbf{Indoor Scenes}} & \multicolumn{4}{c}{\textbf{Outdoor Scenes}} & \multicolumn{4}{c}{\textbf{Combined}} \\ \cline{2-13} 
 & PSNR$\uparrow$ & SSIM$\uparrow$ & MAE$\downarrow$ & LPIPS$\downarrow$ & PSNR$\uparrow$ & SSIM$\uparrow$ & MAE$\downarrow$ & LPIPS$\downarrow$ & PSNR$\uparrow$ & SSIM$\uparrow$ & MAE$\downarrow$ & LPIPS$\downarrow$ \\ \hline
DPDNet\cite{abuolaim2020defocus} & 26.54 & 0.816 & 0.031 & 0.239 & 22.25 & 0.682 & 0.056 & 0.313 & 24.34 & 0.747 & 0.044 & 0.277 \\
KPAC\cite{son2021single} & 27.97 & 0.852 & 0.026 & 0.182 & 22.62 & 0.701 & 0.053 & 0.269 & 25.22 & 0.774 & 0.040 & 0.227 \\
DeepRFT\cite{mao2021deep} & \multicolumn{4}{c}{-} & \multicolumn{4}{c}{-} & 25.71 & 0.801 & 0.039 & 0.218 \\
IFAN\cite{lee2021iterative} & 28.11 & 0.861 & 0.026 & 0.179 & 22.76 & 0.72 & 0.052 & 0.254 & 25.37 & 0.789 & 0.039 & 0.217 \\
DRBNet\cite{ruan2022learning} & \multicolumn{4}{c}{-} & \multicolumn{4}{c}{-} & 25.73 & 0.791 & - & 0.183 \\
Restormer\cite{zamir2022restormer} & 28.87 & {\ul 0.882} & 0.025 & {\ul 0.145} & 23.24 & {\ul 0.743} & 0.050 & {\ul 0.209} & 25.98 & 0.811 & 0.038 & 0.178 \\
EBDB\cite{karaali2017edge} & \multicolumn{4}{c}{-} & \multicolumn{4}{c}{-} & 23.45 & 0.683 & 0.049 & 0.336 \\
DMENet\cite{lee2019deep} & \multicolumn{4}{c}{-} & \multicolumn{4}{c}{-} & 23.41 & 0.714 & 0.051 & 0.349 \\
JNB\cite{shi2015just} & \multicolumn{4}{c}{-} & \multicolumn{4}{c}{-} & 23.84 & 0.715 & 0.048 & 0.315 \\
FocalNet\cite{cui2023focal} & {\ul 29.10} & 0.876 & {\ul 0.024} & 0.173 & {\ul 23.41} & {\ul 0.743} & {\ul 0.049} & 0.246 & {\ul 26.18} & 0.808 & {\ul 0.037} & 0.210 \\
DTPM-4\cite{ye2024learning} & \multicolumn{4}{c}{-} & \multicolumn{4}{c}{-} & 25.98 & {\ul 0.823} & 0.038 & {\ul 0.153} \\ \hline
\textbf{ResFlow(Ours)} & \textbf{29.81} & \textbf{0.907} & \textbf{0.022} & \textbf{0.096} & \textbf{24.25} & \textbf{0.782} & \textbf{0.046} & \textbf{0.166} & \textbf{26.96} & \textbf{0.842} & \textbf{0.034} & \textbf{0.131} \\ \hline
\end{tabular}%
}
\end{table*}

\begin{figure*}[!h]
    \centering
    \includegraphics[width=\textwidth]{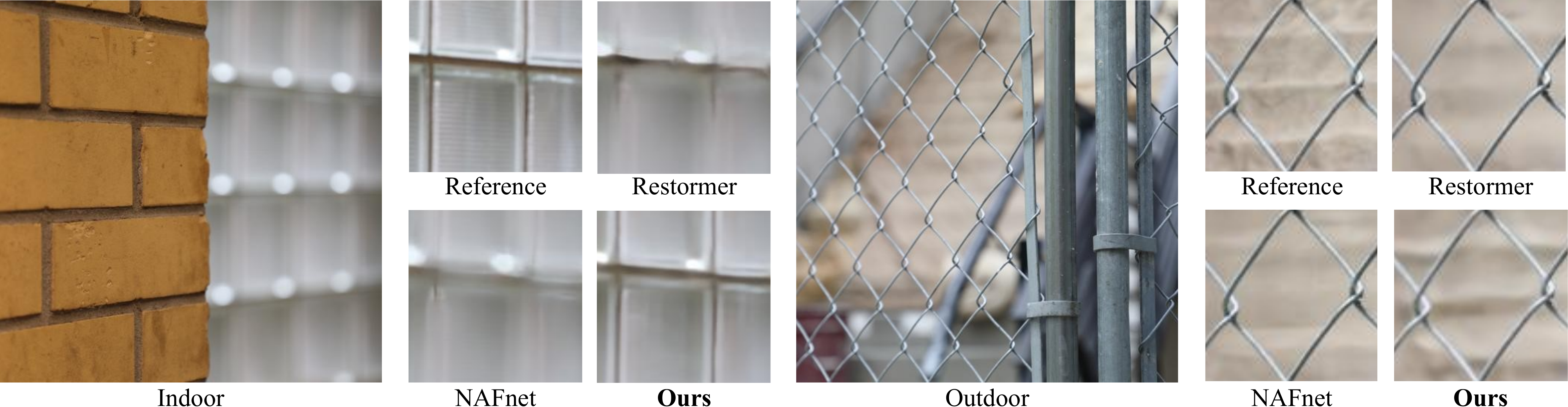}
    \caption{Single-image defocus deblurring results on the DPDD \cite{abuolaim2020defocus} dataset. The part of the image is methodized to observe the local details clearly. From left-top to right-bottom: input blurry images, reference images, and the predicted images obtained by Restormer \cite{zamir2022restormer}, NAFNet \cite{chen2022simple}, and our ResFlow, respectively. See the supplementary material for extra visualizations.}
    \label{fig:dpdd}
\end{figure*}

\noindent
\textbf{Desnowing Results.}
We report our method's quantitative performance on synthetic and real-world datasets in \cref{table:synthetic,table:real}. Overall, our method outperforms state-of-the-art algorithms across all datasets. Specifically, on the synthetic Snow100K-L \cite{liu2018desnownet} dataset, our method surpasses DTPM \cite{ye2024learning} by 0.86dB in PSNR. Additionally, our approach performs well in more challenging real-world scenarios, achieving the best performance across most metrics. On the RealSnow \cite{zhu2023learning} dataset, our method exceeds NAFNet \cite{chen2022simple} by 0.21dB in PSNR. Visual examples of several methods are provided in \cref{fig:weather}, where our method demonstrates superior removal of snowflakes and enhanced detail quality compared to other approaches.
Crucially, our method successfully removes all the snowflakes dispersed in the whole image, while the compared methods leave out some snowflakes in their predicted images. We attribute this to our method's ability to perform multi-step restoration: single-step prediction can be imprecise when the degradation is strong, while our multi-step method can gradually remove the degradation, each step generating better results than in the previous steps.

\noindent
\textbf{Deraining Results.}
The numerical results for the synthetic Outdoor-Rain \cite{li2019heavy} dataset and the real-world LHP \cite{guo2023sky} dataset are presented in \Cref{table:synthetic,table:real}. Our model exhibits strong deraining capabilities, achieving better or comparable results across all metrics. Our method demonstrates a significant 1.83dB improvement in PSNR over NAFNet \cite{chen2022simple} on the Outdoor-Rain dataset. Moreover, our method outperforms SCD-Former \cite{guo2023sky} on the more challenging real-world dataset by 0.20dB in PSNR. Visual results in Fig. \ref{fig:weather} show that our method produces high-quality images resembling ground-truth images with no artifacts.
In particular, our restored images faithfully retain more details, such as the bricks of the walls, than the compared methods, while the compared methods tend to produce smoothed appearances. When the correspondence between HQ and LQ images is ambiguous, one-step estimation converges to the mean of the HQ images conditioned on the LQ images. In contrast, our method introduces the auxiliary variable to disambiguate the HQ images and is able to preserve the sharp details.

\noindent
\textbf{Dehazing Results.}
As shown in Table \ref{table:synthetic}, our method surpasses FocalNet \cite{cui2023focal} by 0.05dB in PSNR on the synthetic Dense-Haze \cite{ancuti2019dense} dataset. On the more challenging real-world NH-HAZE \cite{ancuti2020nh} dataset (Table \ref{table:real}), our approach achieves a notable 1.01dB improvement in PSNR over FocalNet \cite{cui2023focal}, highlighting our method's superior performance in real-world scenarios. Visual results in Fig. \ref{fig:weather} illustrate that our method effectively removes haze while preserving original details.
Similar to the hazy example in \cref{fig:dpdd}, our model recovers more details than other methods, such as the tree's textures. This again demonstrates the promising performance of reversible flows.

\noindent
\textbf{Real Denoising Results.}
As shown in ~\cref{table:real}, our method achieves the best PSNR/SSIM results on the real-world SIDD cite{abdelhamed2018high} denoising dataset. In particular, our method improves PSNR by 2.24dB over Restormer \cite{zamir2022restormer}, demonstrating its superior performance in real denoising scenarios compared to other state-of-the-art methods.

\noindent
\textbf{Defocus Deblurring Results.}
\cref{table:dpdd} presents the quantitative comparison on the DPDD \cite{abuolaim2020defocus} dataset. Our method achieves state-of-the-art results across all metrics compared to existing algorithms. Specifically, our method surpasses FocalNet \cite{cui2023focal} by 0.78dB in the overall category and significantly outperforms DTPM \cite{ye2024learning} by 0.98dB. Moreover, in both indoor and outdoor scenes, our method exceeds the second-best results by 0.71dB and 0.84dB, respectively. These results demonstrate the superior performance of our method in all scenarios. \cref{fig:dpdd} provides visual comparisons.

\begin{figure*}[htbp]
    \centering
    \includegraphics[width=\textwidth]{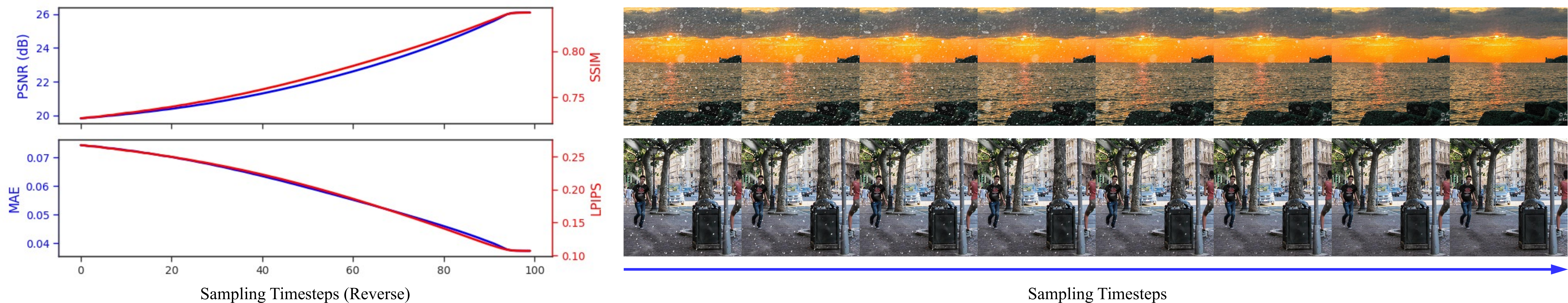}
    \caption{Reserve process performance curves, averaged on 32 samples from the desnowing dataset. The right shows the intermediate results of two example images.}
    \label{fig:abl}
    \vspace{-0.2cm}
\end{figure*}

\begin{table}[htbp]
\caption{Color JPEG compression artifact (QF=10) removal on BSD500 \cite{martin2001database} and LIVE1 \cite{sheikh2005live} datasets.}
\label{table:jpeg}
\renewcommand\arraystretch{1.32}
\resizebox{\columnwidth}{!}{%
\begin{tabular}{lcccc}
\hline
\multirow{2}{*}{\textbf{Method}} & \multicolumn{2}{c}{\textbf{BSD500} \cite{martin2001database}} & \multicolumn{2}{c}{\textbf{LIVE1}\cite{sheikh2005live}} \\ \cline{2-5} 
 & PSNR & SSIM & PSNR & SSIM \\ \hline
QGAC\cite{ehrlich2020quantization} & 27.74 & 0.802 & 27.62 & 0.804 \\
FBCNN\cite{cho2021rethinking} & 27.85 & 0.799 & 27.77 & 0.803 \\
IPT\cite{ozdenizci2023restoring} & 27.57 & 0.792 & 27.37 & 0.799 \\
SwinIR\cite{yang2019scale} & 27.62 & 0.789 & 27.45 & 0.796 \\
DAGN\cite{ma2024sensitivity} & 28.07 & 0.799 & 27.95 & 0.807 \\ \hline
\textbf{ResFlow(Ours)} & \textbf{28.21} & \textbf{0.835} & \textbf{27.95} & \textbf{0.830} \\ \hline
\end{tabular}%
}
\vspace{-0.2cm}
\end{table}

\subsection{Ablation Study}

In this section, we perform an ablation study on the components of our methods, including experiments and analyses about the reverse restoration process, which is the center of image restoration, and the formulation and effects of the auxiliary variable we introduce to enable flow-based degradation modeling. 

\noindent
\textbf{Reverse Restoration Process.}
ResFlow restores images by reversing the degradation process from HQ to LQ images. Specifically, ResFlow progressively removes degradation and noise over several timesteps. We provide specific restoration examples in \cref{fig:abl}, where ResFlow tends to remove degradation along the shortest transport path. Additionally, \cref{fig:abl} shows the performance curve for ResFlow in image desnowing. The results show that deblurring performance improves gradually with the number of steps, converging in the final few steps.

\noindent
\textbf{Reverse Auxiliary Variable.}
The auxiliary variable $\bm{y}_t$ is chosen as a Gaussian distribution at  $t=1$ and zero at $t=0$ at the end of the restoration. We learn the coupling between the auxiliary variable and the uncertainty scope using transport-based generative modeling techniques. By preserving entropy for the reversible process, we provide a schedule for sampling the auxiliary variable at different time steps. Ablation studies on the sampling distribution, schedule, and coupling scheme are presented in \cref{table:abl}. The results show that a Gaussian distribution, which maximizes entropy, helps the model learn the coupling between the auxiliary variable and the uncertainty scope. Additionally, entropy-preserving sampling schedules improve model performance compared to fixed distributions. Finally, we validate different augmentation forms for coupling the auxiliary variable with input, with the generalized augmentation form output by the Adaptor showing the best results.

\begin{table}[t]
\caption{Ablation Study. The average performance on four real-world and synthetic datasets are reported.}
\label{table:abl}
\renewcommand\arraystretch{1.45}
\resizebox{\columnwidth}{!}{%
\begin{tabular}{cc|cc}
\hline
\multicolumn{2}{c|}{\textbf{Method}} & \multicolumn{2}{c}{\textbf{Average}} \\ \hline
\multicolumn{1}{c|}{\textbf{Degradation Schedule}} & \textbf{Auxiliary Variable} & \textbf{PSNR} & \textbf{SSIM} \\ \hline
\multicolumn{1}{c|}{Entropy-preserving} & Gaussian & 25.51 & 0.804  \\
\multicolumn{1}{c|}{Constant} & Guassian & 23.81 & 0.786 \\
\multicolumn{1}{c|}{Entropy-preserving} & Constant & 23.77 & 0.780 \\ \hline
\multicolumn{2}{c|}{\textbf{Injection of Auxiliary Variable}} & \textbf{PSNR} & \textbf{SSIM} \\ \hline
\multicolumn{2}{c|}{Adapter} & 25.51 & 0.804 \\
\multicolumn{2}{c|}{Add} & 24.04 & 0.794 \\
\multicolumn{2}{c|}{Channel-wise concatenation} & 24.94 & 0.798 \\ \hline
\end{tabular}
}
\vspace{-0.2cm}
\end{table}

\section{Conclusion}
\label{sec:conclusion}
\vspace{-0.15cm}

In this paper, we introduced ResFlow, a new framework for image restoration that models the degradation process as a deterministic and invertible flow. By departing from the common practice of stochastic degradation in generative models, ResFlow eliminates unnecessary complexity and allows for a more direct inversion process. By augmenting the degradation process with an auxiliary mechanism, ResFlow disambiguates the uncertainty scope inherent to ill-posed restoration tasks, allowing for fast and efficient inversion of the degradation process. Through extensive experiments, we demonstrated that ResFlow outperforms existing approaches and sets new state-of-the-art, achieving superior restoration quality across several tasks with a few sampling steps. The experimental results across various datasets validate the effectiveness of ResFlow, making it a promising approach for practical image restoration tasks, from weather removal to denoising and beyond. Future work will explore extending ResFlow to more complex degradation models and applying it to video restoration tasks.

\noindent\textbf{Acknowledgments} This work was supported by Beijing Natural Science Foundation (JQ24022), CAAI-Ant Group Research Fund CAAI-MYJJ 2024-02, the National Natural Science Foundation of China (No. 62372451, No. 62192785, No. 62372082, No. 62403462), the Key Research and Development Program of Xinjiang Uyghur Autonomous Region, Grant No. 2023B03024. 

{
    \small
    \bibliographystyle{ieeenat_fullname}
    \bibliography{main}
}

\clearpage
\setcounter{page}{1}
\maketitlesupplementary


\section{Proof of \cref{prop:mi}}

Here, we restate \cref{prop:mi} from \cref{subsec:visual_instructions}.
\begin{customprop}{1}
Given a random process $\{\bm{z}_t \,\vert\, 0 \le t \le 1\}$ defined by
\begin{equation}
    \frac{\partial\bm{z}_t}{\partial t} = \bm{v}(\bm{z}_t, t); \quad 0 \le t \le 1,
    \label{eq:ap_ode}
\end{equation}
where $\bm{v} \in \mathbb{C}^1$ is a velocity field,
denote the mutual information as $\mathrm{MI}(\cdot, \cdot)$, then for any reference random variable $\bm{r}$ and any $0\le t_1, t_2 \le 1$, we have
\begin{equation}
    \mathrm{MI}(\bm{z}_{t_1}, \bm{r}) = \mathrm{MI}(\bm{z}_{t_2}, \bm{r}).
    \label{eq:ap_mi}
\end{equation}
\label{prop:ap_mi}
\end{customprop}

To prove \Cref{prop:ap_mi}, we need the following lemma which states that the mutual information is invariant to invertible maps.
\begin{lemma}[\cite{kraskov2004estimating}]
Given random variables $\bm{X}, \bm{Y}$, suppose the involved mutual information exists and is finite and $\bm{F}, \bm{G}$ are two invertible maps, then
\begin{equation}
    \mathrm{MI}(\bm{X}, \bm{Y}) = \mathrm{MI}(\bm{F}(\bm{X}), \bm{G}(\bm{Y})).
    \label{eq:ap_mi_invariant}
\end{equation}
\label{lemma:ap_mi}
\end{lemma}

See the appendix of \cite{kraskov2004estimating} for a proof of \cref{lemma:ap_mi}.
Now we can prove \cref{prop:ap_mi} using \cref{lemma:ap_mi}.

\begin{proof}[Proof of \cref{prop:ap_mi}]
Suppose \cref{eq:ap_ode} has unique solutions, and the involved mutual information exists and is finite.
Consider the flow $\bm{\Phi}$ of \cref{eq:ap_ode}, that is, $\bm{\Phi}(\bm{z}, s, t)$ solves the following initial-value problem:
\begin{equation}
\begin{cases}
    \frac{\partial\bm{z}_t}{\partial t} = \bm{v}(\bm{z}_t, t), \quad 0 \le t \le 1 \\
    \bm{z}_s = \bm{z}
\end{cases},
\label{eq:ap_ivp}
\end{equation}
Symmetrically, consider the following ODE in reverse time:
\begin{equation}
    \frac{\partial\bm{z}'_t}{\partial t} = -\bm{v}(\bm{z}'_t, 1-t); \quad 0 \le t \le 1,
    \label{eq:ap_ode_reverse}
\end{equation}
its flow is denoted as $\bm{\Psi}(\bm{z}, s, t)$ which solves the following initial-value problem:
\begin{equation}
\begin{cases}
    \frac{\partial\bm{z}'_t}{\partial t} = -\bm{v}(\bm{z}'_t, 1-t), \quad 0 \le t \le 1 \\
    \bm{z}'_s = \bm{z}
\end{cases},
\label{eq:ap_ivp_reverse}
\end{equation}

Consider $\bm{u}_t = [\bm{z}_t^\mathsf{T}, t]^\mathsf{T}$, it satisfies the following autonomous (time-invariant) system:
\begin{equation}
    \frac{\partial}{\partial t}\bm{u}_t
    = \frac{\partial}{\partial t}\left[\begin{matrix}\bm{z}_t\\t\end{matrix}\right]
    = \left[\begin{matrix}\bm{v}(\bm{z}_t, t)\\1\end{matrix}\right]
    = \left[\begin{matrix}\bm{v}(\bm{u}_t)\\1\end{matrix}\right].
    \label{eq:ap_forward_ode}
\end{equation}
Let $\bm{U}$ be the flow of \cref{eq:ap_forward_ode}, then, $\bm{U}(\bm{z}, 0, t) = [\bm{\Phi}^\mathsf{T}(\bm{z}, 0, t), t]^\mathsf{T}$.
Symmetrically, consider $\bm{w}_t = [{\bm{z}'_t}^\mathsf{T}, 1-t]^\mathsf{T}$ satisfying the following autonomous system:
\begin{equation}
    \frac{\partial}{\partial t}\bm{w}_t
    = \frac{\partial}{\partial t}\left[\begin{matrix}\bm{z}'_t\\1-t\end{matrix}\right]
    = \left[\begin{matrix}-\bm{v}(\bm{z}'_t, 1-t)\\-1\end{matrix}\right]
    = -\left[\begin{matrix}\bm{v}(\bm{w}_t)\\1\end{matrix}\right].
    \label{eq:ap_backward_ode}
\end{equation}
Let $\bm{W}$ be the flow of \cref{eq:ap_backward_ode}, then, $\bm{W}(\bm{z}, 0, t) = [\bm{\Psi}^\mathsf{T}(\bm{z}, 0, t), 1-t]^\mathsf{T}$.
Comparing \cref{eq:ap_forward_ode,eq:ap_backward_ode}, it can be seen that they differ only in the sign of the right-hand side. As a result, solving \cref{eq:ap_forward_ode} obtains the \textit{same trajectory in reverse time} as solving \cref{eq:ap_backward_ode} as long as their initial conditions are compatible, that is, $\bm{z}'_s = \bm{\Phi}(\bm{z}, 0, s)$.
Consequently, $\Phi$ and $\Psi$ form a pair of invertible maps for $0 \le s < t \le 1$:
\begin{equation}
    \bm{\Phi}(\bm{\Psi}(\bm{z}, 1-t, 1-s), s, t) = \bm{\Phi}(\bm{\Phi}(\bm{z}, t, s), s, t) = \bm{z}.
    \label{eq:ap_invertible}
\end{equation}

For any reference random variable $\bm{r}$ and any $0\le t_1, t_2 \le 1$, $\bm{G} := \bm{\Phi}(\cdot, t_1, t_2)$ is an invertible map according to \cref{eq:ap_invertible}. Take $\bm{G}$ to be identity map, $\bm{X}=\bm{r}$ and $\bm{Y} = \bm{z}_{t_1}$ in \cref{lemma:ap_mi} gives \cref{eq:ap_mi_invariant}.
\end{proof}

\section{Derivation of Entropy-Preserving Degradation Schedule}

In \cref{subsec:parameterization}, we propose the following entropy-preserving degradation schedule:
\begin{align}
    \alpha_t^{\bm{x}} &= 1 - t,  \quad \sigma_t^{\bm{x}} = 1 - \alpha_t^{\bm{x}} \\ 
    \sigma_t^{\bm{y}} &= \beta \cdot \left( 1 - t + \beta \right)^{-1},  \quad \sigma_t^{\bm{y}} = 1 - \alpha_t^{\bm{y}},
    \label{eq:ap_schedule}
\end{align}
where $\beta=10$ is a hyperparameter.
The intuition is that the entropy should remain constant for a reversible process. For a discrete state space, this intuition holds because, according to \cref{prop:mi}, for all $0 \le s < t \le 1$ we have:
\begin{align}
    H(\bm{z}_s) 
    &= \mathrm{MI}(\bm{z}_s, \bm{z}_s) & \text{(Property of entropy)} \\
    &= \mathrm{MI}(\bm{z}_t, \bm{z}_s) & \text{($\bm{r}=\bm{z}_s$ in \cref{eq:ap_mi})} \\
    &= \mathrm{MI}(\bm{z}_s, \bm{z}_t) & \text{(Symmetry of MI)} \\
    &= \mathrm{MI}(\bm{z}_t, \bm{z}_t) & \text{($\bm{r}=\bm{z}_t$ in \cref{eq:ap_mi})} \\
    &= H(\bm{z}_t)  & \text{(Property of entropy)}
    .
    \label{eq:equal_entropy}
\end{align}
As approximations, assume that
1) HQ images follow uniform distribution $\mathcal{U}$ on $[0, 1]$, which is the maximum-entropy distribution on the normalized pixel range;
2) LQ images follow the Dirac distribution at $0$, which can be seen as the extreme of image degradations.
Then, according to \cref{eq:augment,eq:interpolation_1,eq:interpolation} and $\bm{y}_0 = \bm{0}, \bm{y}_1 \sim \mathcal{N}(\bm{0}, \bm{I})$, we have
\begin{align}
    H(\bm{z}_t) 
    &= H(\bm{x}_t) + H(\bm{y}_t) \\
    &= H(\alpha_t^{\bm{x}}\bm{x}_0 + \sigma_t^{\bm{x}}\bm{x}_1) + H(\alpha_t^{\bm{y}}\bm{y}_0 + \sigma_t^{\bm{y}}\bm{y}_1) \\
    &= H((1-t)\bm{x}_0) + H(\sigma_t^{\bm{y}}\bm{y}_1) \\
    &= d \ln (1-t) + \frac{d}{2} (1 + \ln (2 \pi)) + d \ln \sigma_t^{\bm{y}}
    ,
    \label{eq:ap_entropy_t}
\end{align}
where $d$ is the channel dimension.
The last equality comes from the fact that the entropy of $\mathcal{U}[a, b]$ is $\ln|b-a|$ and that the entropy of $\mathcal{N}(\bm{\mu}, \bm{\Sigma})$ is $\frac{d}{2} (1 + \ln (2 \pi)) + \frac{1}{2} \ln |\bm{\Sigma}|$ \cite{kraskov2004estimating}.
Similarly,
\begin{align}
    H(\bm{z}_s) 
    = d \ln (1-s) + \frac{d}{2} (1 + \ln (2 \pi)) + d \ln \sigma_s^{\bm{y}}
    .
    \label{eq:ap_entropy_s}
\end{align}
Plugging \cref{eq:ap_entropy_t,eq:ap_entropy_s} into \cref{eq:equal_entropy} gives
\begin{align}
    \ln (1-t) + \ln \sigma_t^{\bm{y}}
    = \ln (1-s) + \ln \sigma_s^{\bm{y}}
    .
    \label{eq:ap_entropy_st}
\end{align}
Let $s=0$ and $\sigma_0^{\bm{y}}=\beta > 0$ in \cref{eq:ap_entropy_st} gives the form of $\sigma_t^{\bm{y}}$ as:
\begin{align}
    \sigma_t^{\bm{y}}
    = \beta/(1-t)
    .
    \label{eq:ap_entropy_sigma_origin}
\end{align}
\cref{eq:ap_entropy_sigma_origin} is singular at $t=1$, so we introduce $\beta$ into the denominator of \cref{eq:ap_entropy_sigma_origin}, obtaining the final form in \cref{eq:ap_schedule}. \cref{fig:schedule} shows the curves of the degradation schedule $\sigma_t^{\bm{y}}$ under various $\beta$ values. We fix $\beta=10$ for all experiments without tuning.

\begin{figure}[h]
    \centering
    \includegraphics[width=\linewidth]{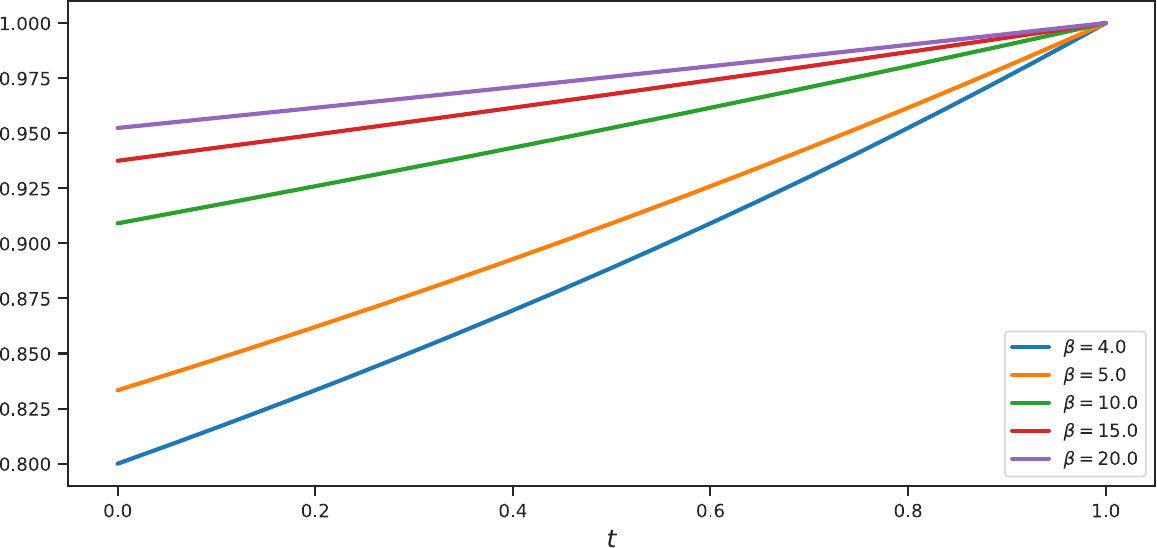}
    \caption{Degradation schedule $\sigma_t^{\bm{y}}$ with various $\beta$ values.}
    \label{fig:schedule}
\end{figure}

\cref{table:decay_schedule} compares our degradation schedule with the constant schedule and the linearly decaying schedule. Our degradation schedule is consistently better than both alternatives, demonstrating the empirical effectiveness of our approach.

\begin{table}[htbp]
\centering
\caption{Experiments about degradation schedules. The average performance on deraining/desnowing/denoising datasets is reported.}
\label{table:decay_schedule}
\begin{tabular}{c|cc}
\hline
\textbf{Method} & PSNR & SSIM \\ \hline
Constant & 35.69 & 0.943 \\
Linear & 36.25 & 0.947 \\
Ours & 36.82 & 0.949 \\ \hline
\end{tabular}
\end{table}

\section{Datasets}

We evaluate the performance of image restoration methods on five major image restoration tasks, including desnowing, draining, dehazing, denoising, and JPEG compression artifact removal, using synthetic and real-world datasets. Details of the datasets are given below according to their corresponding tasks:

\vspace{0.5em}
\noindent\textbf{Image Desnowing}:
Snow100K \cite{liu2018desnownet} is a synthetic snow removal public dataset that includes synthetic snow images and corresponding snow-free GT images. The simulated snowflake particles contain a variety of different nozomura, and also have different densities, shapes, trajectories, and transparency in order to add variation. We used Snow100K-l, which has the highest level of diversity, for the evaluation method.
Snow100K \cite{liu2018desnownet} contains 50000 images for training and 50000 images for testing.
RealSnow \cite{zhu2023learning} is a real-world snow removal dataset that acquires image pairs from background-static video. These images feature a variety of urban and natural background scenes that contain varying densities of snowfall and illuminations.
RealSnow \cite{zhu2023learning} contains 61500 (crops) and 240 training and testing images.

\vspace{0.5em}
\noindent\textbf{Image Deraining}:
Outdoor-rain \cite{li2019heavy} is a synthetic rain removal dataset, which render synthetic rain  streaks and rain accumulation effects based on the provided depth information. These effects include the veiling effect  caused by the water particles, as well as image blur. Outdoor-Rain is a set of outdoor rainfall datasets created on clean outdoor images.
Outdoor-Rain \cite{li2019heavy} contains 8100 images for training and 900 images for testing;.
LHP \cite{guo2023sky} is a real-world rain removal dataset. Real image pairs are acquired by keeping the camera motionless to record real rain videos with static backgrounds, which contains a variety of rainfall patterns rain a variety of typical scenarios.
LHP \cite{guo2023sky} contains 300 images for testing.

\vspace{0.5em}
\noindent\textbf{Image Dehazing:}
Dense-Haze \cite{ancuti2019dense} is a synthetic defogging dataset with dense and homogeneous haze scenes. It contains haze images and corresponding clean images of various outdoor scenes.
Dense-Haze \cite{ancuti2019dense} contains 49 images for training and 6 images for testing. 
NH-HAZE \cite{ancuti2020nh} is a realistic image dehazing dataset with non-homogeneous hazy and haze-free paired images. The non-homogeneous haze has been generated using a professional haze generator that imitates the real conditions of haze scenes. It contains various outdoor scenes.
NH-HAZE \cite{ancuti2020nh} contains 49 images for training and 6 images for testing, totaling 55 outdoor scenes.

\vspace{0.5em}
\noindent\textbf{Real Denoising:}
SIDD \cite{abdelhamed2018high} is a realistic denoising dataset, which captured real noisy images using five representative smartphone cameras and generated their ground truth images.
SIDD \cite{abdelhamed2018high} contains 288 images for training and 32 images for testing.

\vspace{0.5em}
\noindent\textbf{Defocus Deblur:}
DPDD \cite{abuolaim2020defocus} is a synthetic defocus debluring dataset, which capture a pair of images of the same static scene at two aperture sizes which are the maximum (widest) and minimum (narrowest) apertures possible for the lens configuration.
Focus distance and focal length differ across captured pairs in order to capture a diverse range of defocus blur types.
DPDD \cite{abuolaim2020defocus} contains 350 images for training and 76 images for testing.

\vspace{0.5em}
\noindent\textbf{JPEG Artifact removal:}
The training dataset is collected from DIV2K and FLICKR2K \cite{Agustsson_2017_CVPR_Workshops} containing 900 and 2650 images, respectively. The testing dataset includes LIVE1 \cite{sheikh2005live} which contains 29 testing images, and BSD500 \cite{arbelaez2010contour} which contains 500 testing images.



\section{Implementation Details}

\begin{figure}[h]
    \centering
    \includegraphics[width=\linewidth]{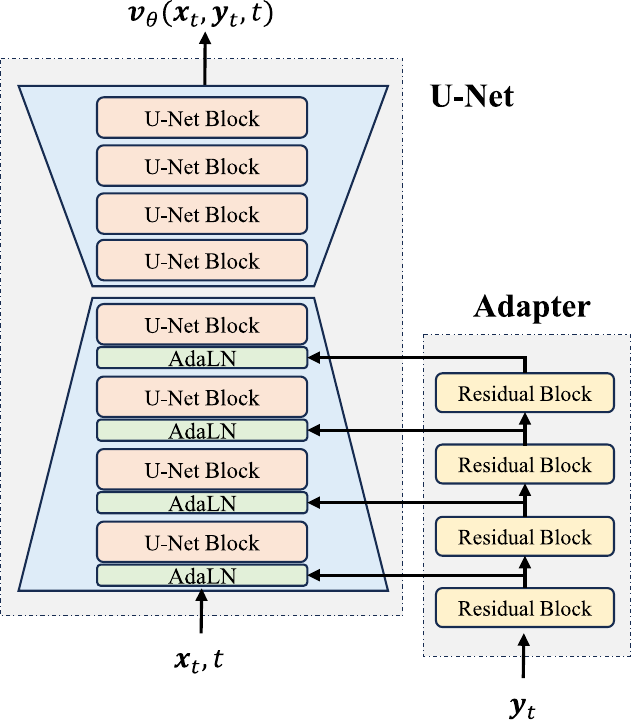}
    \caption{Model architecture.}
    \label{fig:sup_model}
\end{figure}

All experiments adopt the same U-Net architecture from \cite{ho2020denoising} as the backbone.
The input to the U-Net is the $\bm{x}_t$ starting from $\bm{x}_1$ as the LQ image. The output of the U-Net is the velocity $\bm{v}_\theta(\bm{x}_t, \bm{y}_t, t)$.
We remove the class-label conditioning and condition the model on $\bm{y}_t$ via an adapter \cite{mou2024t2i} as illustrated in \cref{fig:sup_model}.
The adapter processes $\bm{y}_t$ with a stack of residual blocks \cite{he2016deep}, each downsampling the feature maps of $\bm{y}_t$ to align with the spatial size of the feature maps of the corresponding in the U-Net. The feature maps are fused with the corresponding feature maps of the U-Net by AdaLN \cite{huang2017arbitrary}. The output layer of each residual block is initialized to zero, so the overall model is equivalent to the U-Net initially.

The model is trained with a batch size of 8. We use the AdamW optimizer \cite{loshchilov2017decoupled} with $\beta_1 = 0.9$ and $\beta_2 = 0.999$. The initial learning rate is 1e-4 and decayed to 1e-6 via cosine annealing \cite{loshchilov2016sgdr}. The input images are normalized to the range of $[-1, 1]$ and randomly cropped to $256\times256$ during training. The model is trained on each dataset with eight NVIDIA A100 GPUs for 400K iterations.

\section{Discussion of Different Approaches}

Image restoration is modeled as paired image-to-image translation from LQs to HQs. Historically image restoration methods are \textit{one-step} models that map LQs to HQs with one single step of inference. Diffusion-based models, however, are \textit{multi-step} models that map refine LQs in one or more iterative steps.
Similar to diffusion-based methods, our ResFlow allows multistep image restoration that both improves the quality of restored HQ and boosts training (because each step only requires partially restoring the image).
Different from diffusion-based methods, ResFlow models HQ-to-LQ degradation by deterministic paths (normalizing flows), which is both easier to learn and more efficient to sample from LQ to HQ.
CDPMS \cite{niu2023cdpmsr}, RDDM \cite{liu2024residual}, etc. extend diffusion models by modifying the starting point or diffused variables (e.g. introducing residuals). However, they still suffer from low training and sampling efficiency due to stochastic degradation paths of diffusion models. Our ResFlow introduces deterministic degradation paths to solve this problem and achieves superior performance.
Note that all the methods we discuss and compare with, including our ResFlow, are learned by supervised training using paired LQ and HQ images. Learning image restoration models with unpaired images is still an open problem.

An interesting method that draws similarity with our ResFlow is Cold Diffusion \cite{bansal2023cold}, which tackles image generation. Cold Diffusion also inverts the degradations applied to the images that it is trained to generate (similar to the HQ images). However, the crucial difference between Cold Diffusion and ResFlow is that Cold Diffusion only requires the generated images to be natural, while ResFlow focuses on LQ-to-HQ restoration and also require HQ to preserve LQ's information. The ``generation paths'' of Cold Diffusion are bounded only on one end; but the degradation paths of ResFlow (\Cref{eq:ode}) are bounded by both LQ and HQ.
Put the difference in implementation, during training, during training, Cold Diffusion learns to estimate the noise-free images; while ResFlow learns to estimate the velocity of degradation flow and introduces an auxiliary variable to ensure reversibility (\Cref{eq:optimization_inner,eq:optimization}).
During inference, Cold Diffusion iteratively estimates noise-free images and degrades them with less intensity; while we solve \Cref{eq:ode,eq:augment} from $t=1$ to $0$ by Euler integration (that is, accumulating velocity $\times$ step size).

\section{Addition Experimental Results}

\noindent\textbf{Computational costs.} Diffusion models are notoriously slow because they require dozens or even hundreds of inference steps. However, our ResFlow can generate high quality restored in as few as two or even one step. Compared with diffusion-based models such as WeatherDiff \cite{ozdenizci2023restoring}, we consistently achieve better performance (e.g., \textbf{32.82} vs 28.38 PSNR$\uparrow$ for deraining) with a significantly lower computational cost of \textbf{592.44} vs. 2634.8 GFLOPs$\downarrow$ and \textbf{420.8s} vs. 2488.8s latency$\downarrow$.

\noindent\textbf{Extra visualizations.}
We provide more visualization results on the synthesized and real-world datasets as shown in \cref{fig:fig_suppl1,fig:fig_suppl2}. Synthesized datasets contain Desnowing, Deraining, Dehazing, and Single-image Defocus Deblurring results on Snow100K \cite{liu2018desnownet}, Outdoor-Rain \cite{li2019heavy}, Dense-Haze \cite{ancuti2019dense}, and DPDD \cite{abuolaim2020defocus} datasets. Real-world datasets contain Dehazing results on NH-HAZE 
\cite{ancuti2020nh}, Denoising results on SIDD \cite{abdelhamed2018high}, Deraining results on LHP \cite{guo2023sky}, and Desnowing results on RealSnow \cite{zhu2023learning}.
Extra visualizations on DPDD \cite{abuolaim2020defocus} is shown in \cref{fig:flow_sup_cr1}, our method significantly outperforms Restormer \cite{zamir2022restormer} perceptually.

\Cref{fig:flow_sup_cr2} shows the impact of auxiliary variables on the generated results. After optimizing \Cref{eq:optimization_inner,eq:optimization}, ResFlow learns a deterministic coupling from the joint distribution of the auxiliary variable and LQ, to that of the HQ. Conceptually, the auxiliary variable is mapped to the ``information difference'' between HQ and LQ. When there are multiple possible HQs for an LQ, sampling an auxiliary variable and evaluating \Cref{eq:parameterization} will produce a unique velocity, leading to a unique HQ, thus disambiguating the velocity and HQ.
\Cref{fig:flow_sup_cr2} is an example of how different auxiliary variables (Aux.) lead to different HQ via different velocities, where blue boxes highlight the differences.

\begin{figure*}[ht]
    \centering
    \includegraphics[width=0.88\textwidth]{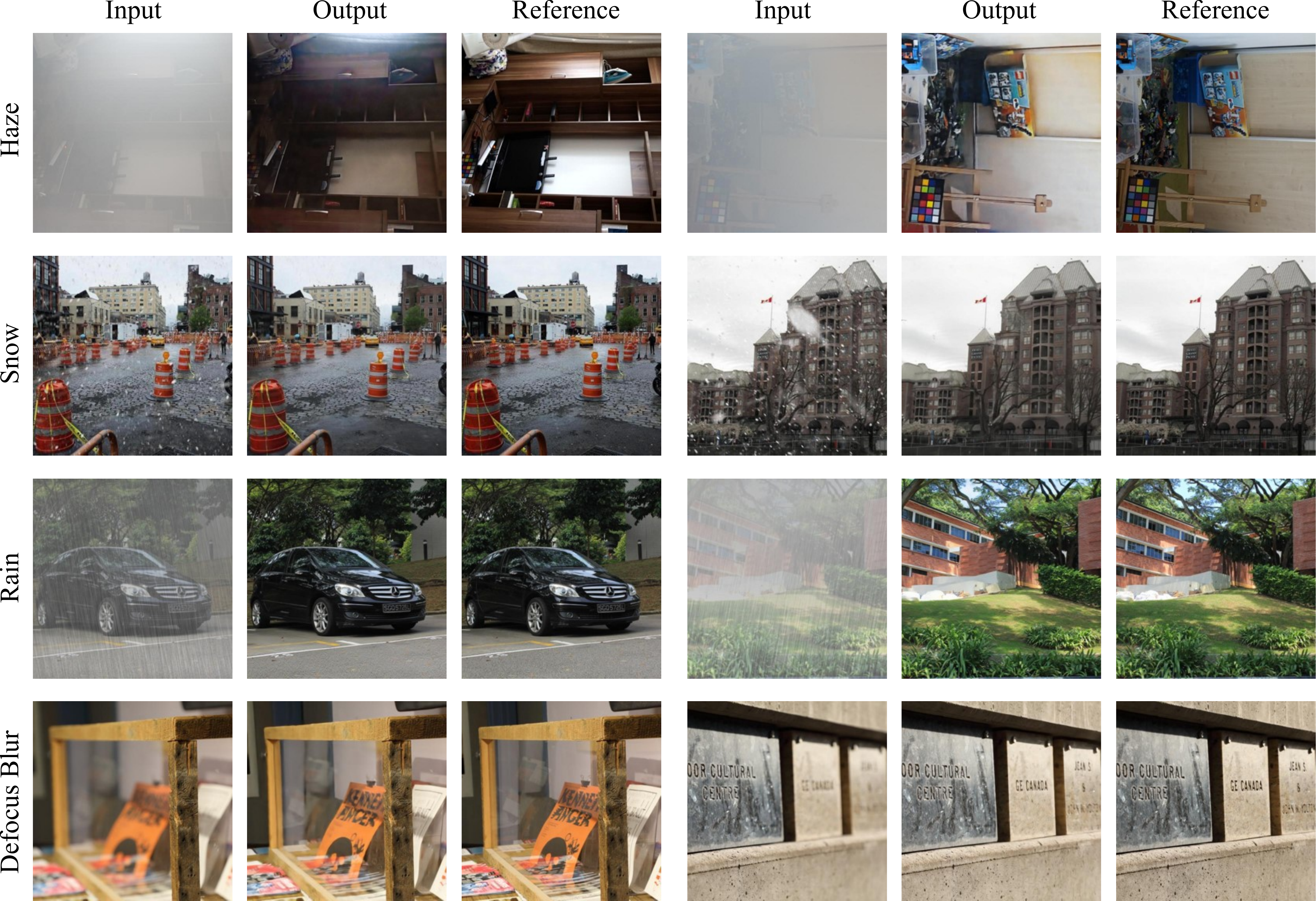}
    \vspace{-0.5em}
    \caption{Visual results of synthesized datasets.}
    \label{fig:fig_suppl1}
\end{figure*}

\begin{figure*}[ht]
    \centering
    \includegraphics[width=0.88\textwidth]{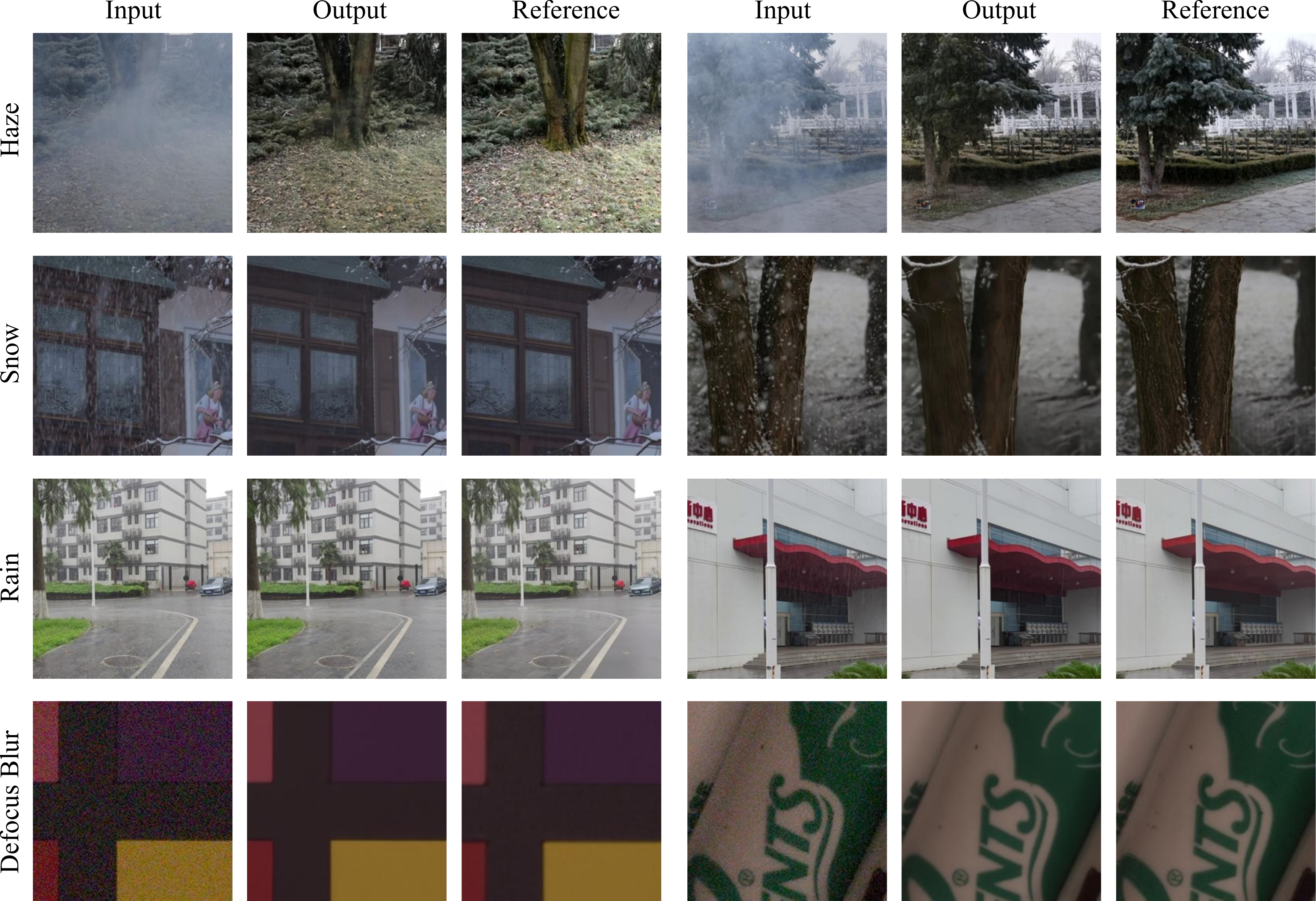}
    \vspace{-0.5em}
    \caption{Visual results of real-world datasets.}
    \label{fig:fig_suppl2}
\end{figure*}

\begin{figure*}[!h]
    \centering
    \includegraphics[width=\textwidth]{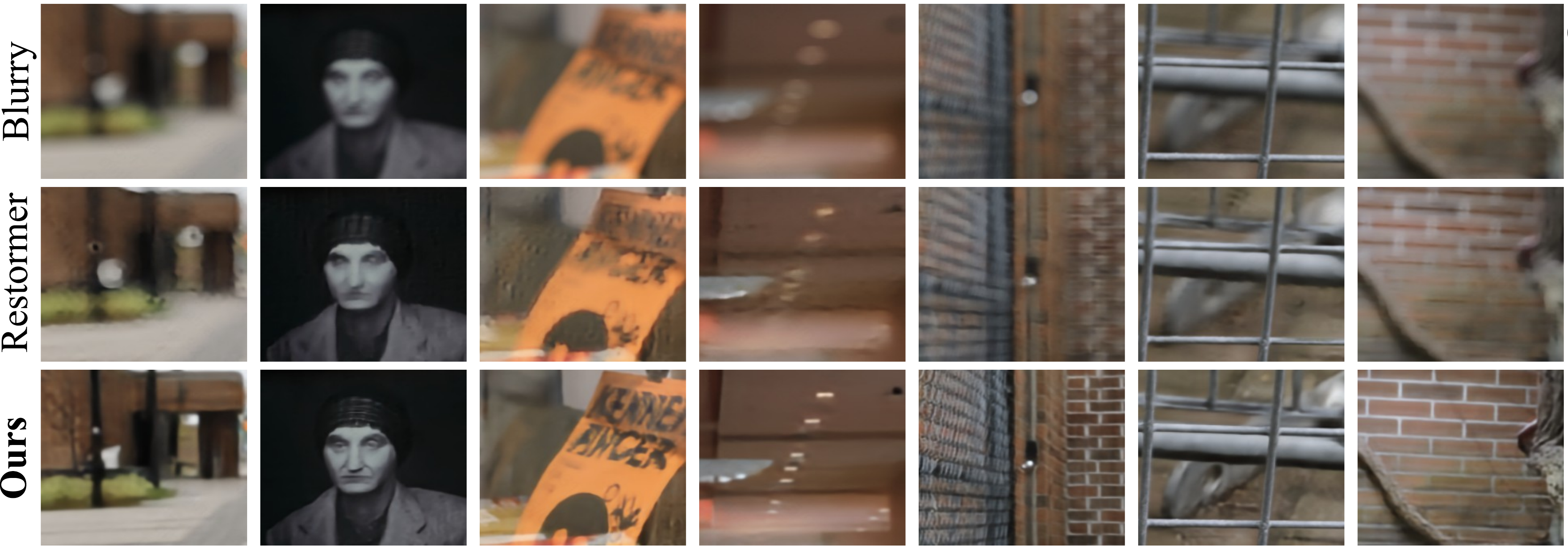}
    \caption{Extra single-image defocus deblurring results on the DPDD \cite{abuolaim2020defocus} dataset. The part of the image is methodized to observe the local details clearly. From top to bottom: input blurry images, the predicted images obtained by Restormer \cite{zamir2022restormer} and our ResFlow.}
    \label{fig:flow_sup_cr1}
\end{figure*}

\begin{figure*}[!h]
    \centering
    \includegraphics[width=\textwidth]{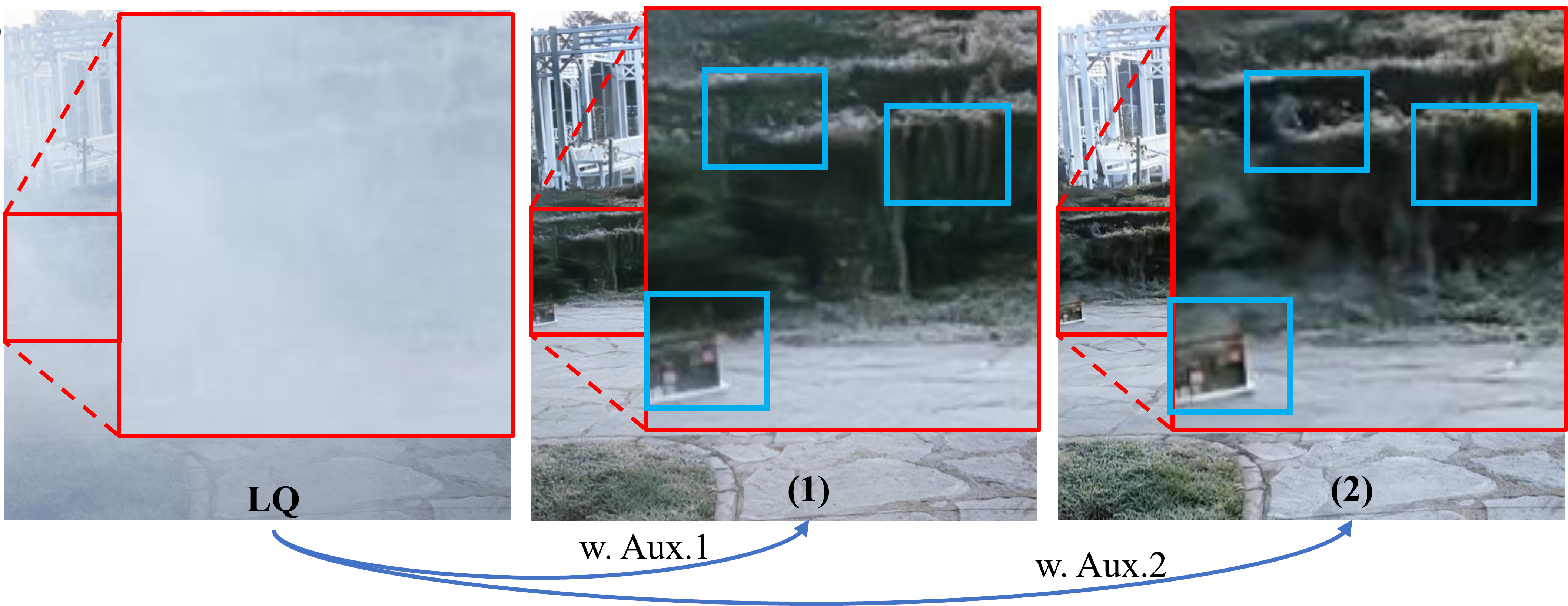}
    \caption{Extra single-image defocus deblurring results on the DPDD \cite{abuolaim2020defocus} dataset. The part of the image is methodized to observe the local details clearly. From top to bottom: input blurry images, the predicted images obtained by Restormer \cite{zamir2022restormer} and our ResFlow.}
    \label{fig:flow_sup_cr2}
\end{figure*}

\end{document}